\documentclass{article}[accepted]

\usepackage{microtype}
\usepackage{graphicx}
\usepackage{subcaption}
\usepackage{booktabs} 

\usepackage{hyperref}

\usepackage{soul}
\usepackage{amsthm}
\usepackage{multirow}

\usepackage[accepted]{icml2026}

\usepackage{amsmath}
\usepackage{amssymb}
\usepackage{mathtools}
\usepackage{amsthm}

\usepackage[capitalize,noabbrev]{cleveref}

\theoremstyle{plain}
\newtheorem{theorem}{Theorem}[section]

\theoremstyle{definition}
\newtheorem{definition}[theorem]{Definition}

\theoremstyle{remark}

\usepackage[textsize=tiny]{todonotes}

\icmltitlerunning{Watch Your Step: Information Injection in Diffusion Models via Shadow Timestep Embedding}

\begin{document}

\twocolumn[
  \icmltitle{Watch Your Step: Information Injection in Diffusion Models \\ via Shadow Timestep Embedding}



  \icmlsetsymbol{equal}{*}

  \begin{icmlauthorlist}
  \icmlauthor{An Huang}{yyy}
  \icmlauthor{Junggab Son}{yyy}
    \icmlauthor{Zuobin Xiong}{yyy}
  \end{icmlauthorlist}

  \icmlaffiliation{yyy}{Department of Computer Science, University of Nevada Las Vegas, Las Vegas, USA}

  \icmlcorrespondingauthor{Zuobin Xiong}{zuobin.xiong@unlv.edu}


  \vskip 0.3in
]



\printAffiliationsAndNotice{}  

\begin{abstract}
Diffusion models have become the foundation of modern generative systems, with most research focusing primarily on improving generation efficiency and output quality.
The timestep embedding component is a crucial part of the diffusion pipeline, which provides a temporal conditioning signal to the denoising network, enabling it to adapt its predictions across different noise levels throughout the process.
Despite their potential to contain substantial information, timestep embeddings remain underexplored in current research, especially for security risks and reliable provenance. 
To fill this gap, we introduce \textbf{Shadow Timestep Embedding (STE)}, a novel mechanism that investigates the underutilized temporal space for malicious information injection into diffusion models.
In particular, when zooming in on the timestep embedding space, we find that different timesteps exhibit distinct representational capabilities that can encode side-channel information.
Moreover, such encoded information can be utilized for attack and defense purposes through the scheduler interface. 
We present a theoretical analysis of timestep embeddings as position-encoding mappings and derive a mutual coherence evaluation that explains the separability of disjoint timestep intervals.
Our findings reveal the diffusion model's timestep as a powerful side channel for carrying dedicated information, motivating new directions for adversarial generative modeling by understanding the temporal dimension.
\end{abstract}

\section{Introduction}
\label{sec:intro}

Trained and fine-tuned on large-scale datasets~\cite{schuhmann2022laion, schuhmann2021laion}, diffusion models (DMs)~\cite{ho2020denoising,song2020denoising,dhariwal2021diffusion,ho2022cascaded,ho2022classifier,rombach2022high,zhang2022fast} have been utilized as the standard backbone for high-fidelity content generation across images, video, audio, and natural language, achieving state-of-the-art generation quality through iterative denoising.
Although early diffusion models, e.g., DDPM, require thousands of denoising steps, recent advances in sampling efficiency have reduced the number of sampling steps by one to two orders of magnitude~\cite{lu2022dpm,zhao2023unipc}, accelerating real-world deployment across consumer and enterprise pipelines.
However, this growing ubiquity of diffusion models raises urgent questions about safety, accountability, and provenance~\cite{guo2025copyrightshield, carlini2023extracting,duan2023diffusion,truong2025attacks}, as the powerful AI tool enables precise generative control, which can be misused to produce harmful content~\cite{zhang2024generate}.
Therefore, understanding the comprehensive threat surfaces and security implications in diffusion pipelines and how they can be subverted has become an important research topic.

Early works~\cite{yu2023cross,chen2025hiding} show that diffusion models can be used as particularly powerful steganographic carriers, in which auxiliary information is hidden within model parameters or latent representations via steganography~\cite{younis2025attenhidenet, sanjalawe2025deep} while preserving the perceptual fidelity of generated data.
These findings suggest that diffusion models support both output-level and model-level secret channels to \texttt{embed information.}
On the other hand, recent security studies demonstrate that diffusion models are vulnerable to a set of attacks~\cite{chen2023advdiffuser,shan2024nightshade,chou2023backdoor}, especially backdoor attacks~\cite{lin2025backdoordm}.
In such a scenario, the attacker can implant backdoor malicious triggers into diffusion models based on steganography (i.e., \texttt{embedding malicious information}), so that crafted prompts can produce undesired outputs at inference time~\cite {chou2023villandiffusion,zhai2023text,chen2025invisible,chen2025parasite}.
On the contrary, steganography can also serve as a watermark mechanism for model attribution and content traceability by injecting signed embeddings and verifying their presence in the model output~\cite{li2025dnnkeylock,wang2025cascade}.

So far, the mainstream methods in information injection in diffusion models focused on conditions and output (e.g., generated images) aspects, while the intrinsic control over the timestep embedding of the model is overlooked, which could exhibit significant potential in different applications.
This very research gap motivates us to study whether the temporal embeddings in diffusion models can function as a covert information channel, enabling isolated generative behaviors without modifying the observable sampling procedure.

Moreover, through literature review, we found that timestep security is a pressing and practical issue in diffusion models because they are inconspicuous yet essential.
For instance, as shown in Fig.~\ref{fig:threat}, an adversary can exploit the temporal embedding space by subtly encoding information in the timestep embeddings.
Such malicious manipulation can be achieved through existing code poisoning attacks~\cite{gokkaya2026software,wan2024data}, where the users mistakenly import disguised pipelines that are published by attackers. 
Therefore, these methods can become stealthier, leave fewer fingerprints, and easily evade defenses that monitor input/output spaces in traditional settings.
\begin{figure}
    \centering
    \includegraphics[width=0.9\linewidth]{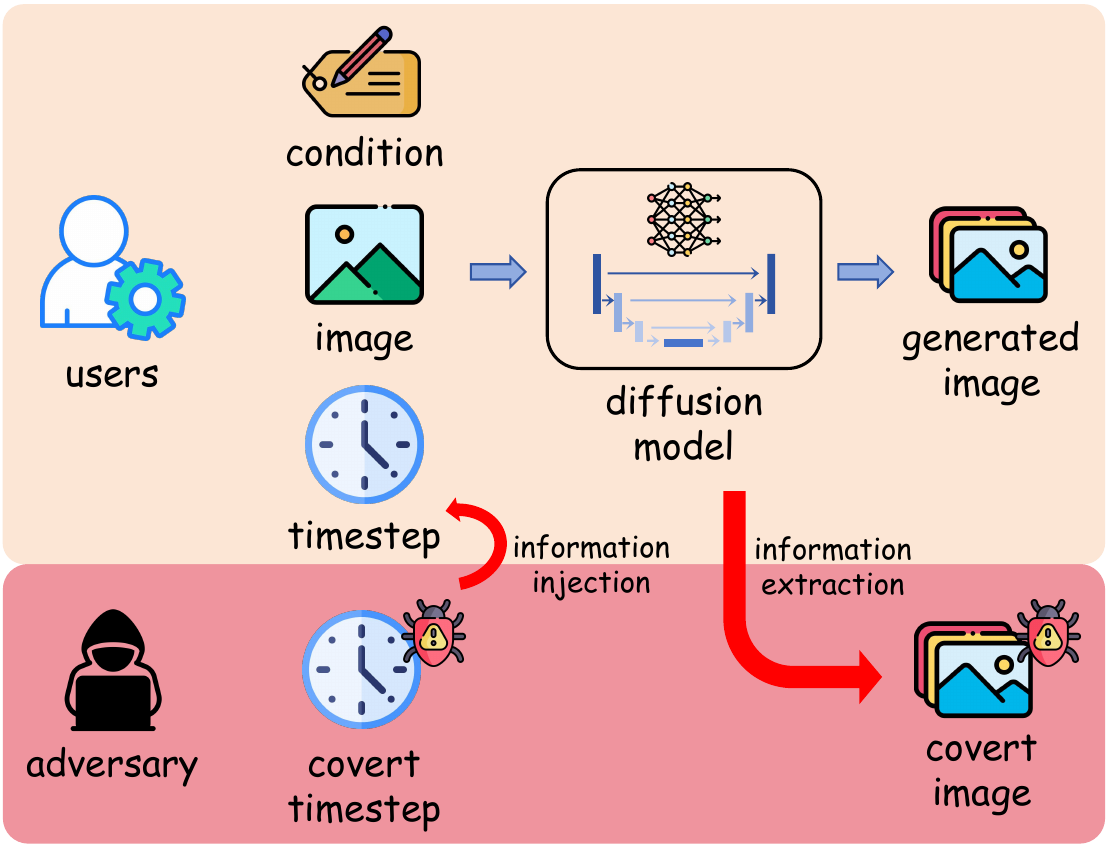}
    \caption{Adversaries can use the timestep as a secret channel, which can inject extra information. }
    \label{fig:threat}
\end{figure}
Based on the application, this work introduces a novel information injection method, Shadow Timestep Embedding (\textbf{STE}), which
operates on the temporal side-channel, introducing an invisible yet controllable timestep embedding interface.

Specifically, 
we find that modifying the timestep range in diffusion models can extend the embedding resolution, providing an unoccupied subspace for more information injection.
Our analysis illustrates that the injected information can be malicious or legitimate, which can be extended to a broader security scenario beyond steganography. 

The contributions of this work are as follows. 
\begin{itemize}
    \item We propose Shadow Timestep Embedding (STE), a timestep-based information injection method for diffusion models, which uncovers an underutilized temporal channel to encode additional, controllable information.

    \item We showcase that the extended embedding space can introduce a dual-use security surface, e.g., STE can function as a covert attack injection method or as a watermark verification tool.
    
    \item Through theoretical analysis, we prove the mutual coherence between different timesteps, revealing the fundamental reason that the embedding space can hold representation.
    
    \item Experiment results highlight that STE can inject auxiliary data distributions reliably while maintaining independence between the explicit and shadow manifolds. 
\end{itemize}

\section{Related works}
\label{sec:related}

\paragraph{Steganography in Diffusion Models.}
Recent work has explored diffusion models as powerful carriers for steganography, leveraging the denoising process to embed and recover hidden information. 
StegaDDPM~\cite{peng2023stegaddpm} embed secret messages into the denoising trajectory or noise space, achieving high-capacity and visually imperceptible steganography. 
Training-free approaches such as CRoSS~\cite{yu2023cross} introduce controllable and secure steganographic mechanisms by explicitly conditioning the diffusion process. 
Complementary to output-level hiding, DMIH~\cite{chen2025hiding} demonstrates that diffusion models themselves can act as steganographic containers, embedding hidden image mappings directly into the learned score function.

However, prior approaches do not fully exploit the representational capacity of the temporal timestep embedding, which offers a covert and practical mechanism for information injection and extraction.

\paragraph{Security of Diffusion Models.} 
The rapid deployment of diffusion models has raised concerns about their robustness and provenance.
Backdoor attacks demonstrate that diffusion models can be maliciously fine-tuned to produce attacker-chosen outputs under specific triggers while behaving normally otherwise.
VillanDiffusion~\cite{chou2023villandiffusion} provides a unified framework for both unconditional and conditional backdoors, revealing their persistence across different schedulers.
Similarly, BadT2I~\cite{zhai2023text} embeds multi-modal triggers into text-to-image systems with minimal data poisoning.
On the defensive side, concept erasure methods~\cite{gandikota2023erasing,gong2024reliable,chen2024score} attempt to remove harmful or copyrighted content via targeted parameter updates,
whereas watermarking aims to establish content provenance.
Specifically, Tree-Ring~\cite{wen2023tree} encodes reversible frequency-domain signatures along the full sampling trajectory,
and ROBIN~\cite{huang2024robin} leverages adversarial optimization to embed robust, invisible watermarks aligned with diffusion dynamics.

Despite these advances, most security research focuses on data, textual prompts, or global model parameters, leaving open the timestep embedding itself as an under-examined dimension for both attacks and defenses.

\section{Method}
\label{sec:method}

\begin{figure*}
    \centering
    \includegraphics[width=0.85\linewidth]{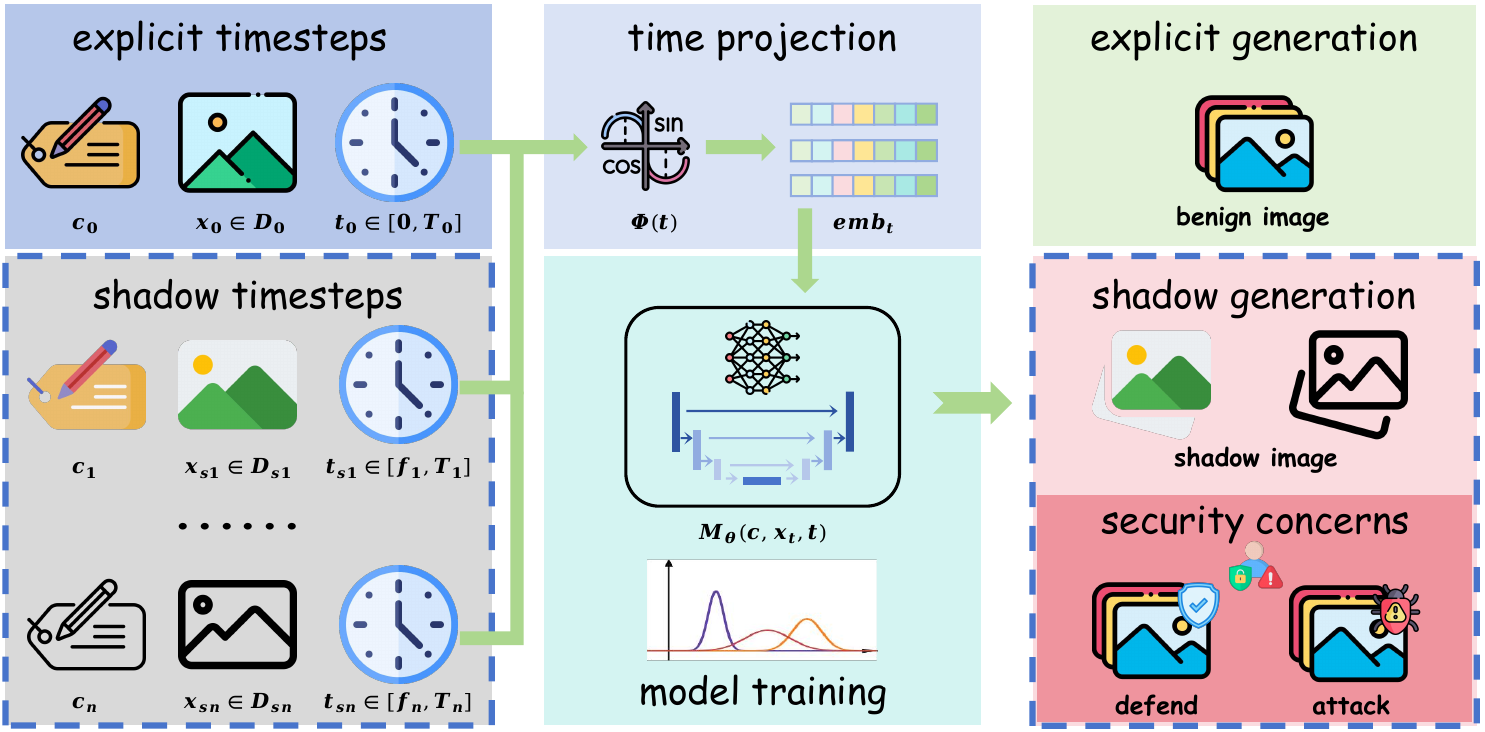}
    \caption{Architecture of Shadow Timestep Embedding (STE). The framework extends conventional diffusion training by introducing shadow timesteps, which parallel explicit timesteps but occupy disjoint temporal intervals.
    Each shadow timestep subset $t_{sn}$ is associated with different data distributions $D_{sn}$.
    During time projection, timesteps are encoded via sinusoidal embeddings and fed into the diffusion model for joint training.
    At inference, the model can generate either explicit or shadow images.
    Such multi-space learning enables new capabilities but also raises security concerns, as shadow subspaces can be exploited for covert defend or attack.}
    \label{fig:architecture}
\end{figure*}

\subsection{Preliminaries}
\textbf{Diffusion Models.}
Diffusion models are a class of generative models that learn to synthesize data by inverting a gradual noising process.
Given a clean sample $\mathbf{x}_0 \sim q(\mathbf{x}_0)$, the forward process progressively perturbs $\mathbf{x}_0$ through a sequence of Gaussian transitions:
{\small\begin{equation}
q(\mathbf{x}_t \mid \mathbf{x}_{t-1}) 
= \mathcal{N}\!\left(
\sqrt{1-\beta_t}\,\mathbf{x}_{t-1},\, 
\beta_t \mathbf{I}
\right), 
\quad t = \{1, \dots, T\},
\end{equation}}%
where $\{\beta_t\}_{t=1}^T$ is a pre-defined variance schedule controlling the noise magnitude at each timestep.

The generative process learns to invert this corruption by predicting the added noise $\boldsymbol{\epsilon}$ at each timestep.
A neural network $\boldsymbol{\epsilon}_\theta(\mathbf{x}_t, t, \mathbf{c})$ parameterized by $\theta$ is trained to approximate the conditional mean of the reverse transition:
{\small\begin{equation}
p_\theta(\mathbf{x}_{t-1} \mid \mathbf{x}_t)
= \mathcal{N}\!\left(
\boldsymbol{\mu}_\theta(\mathbf{x}_t, t, \mathbf{c}),\,
\sigma_t^2\mathbf{I}
\right),
\end{equation}}%
where $\mathbf{c}$ denotes optional conditioning (e.g., text or class label) and $t$ is represented via a \emph{timestep embedding}.
The standard training objective minimizes the denoising error between predicted and true noise:
{\small\begin{equation}
\label{eq:dm}
\mathcal{L}_{\text{DM}}(\theta)
= \mathbb{E}_{t,\mathbf{x}_t,\boldsymbol{\epsilon}}\!
\left[
\left\|
\boldsymbol{\epsilon}
- \boldsymbol{\epsilon}_\theta\!\left(
\mathbf{x}_t,
t,\mathbf{c}
\right)
\right\|_2^2
\right].
\end{equation}}

Each discrete timestep $t$ (typically $t\!\in\![0,1000]$) is mapped to a continuous vector using sinusoidal or learned positional encoding, denoted $\mathbf{emb}_t = \Phi(t)$.
This embedding acts as a global conditioning signal broadcast to every block of the denoising network, effectively coupling the scheduler’s time dynamics with the model’s internal representation.
Our work builds on this observation and extends the embedding range beyond the conventional setup to explore the information capacity of the temporal subspace.

\subsection{Shadow Timestep Embedding}
Standard diffusion models operate over a fixed and compact timestep range $t \in [0, T_0]$, where each $t$ is mapped to a continuous feature vector through a learned or sinusoidal positional encoding. 
This design implicitly assumes that the entire temporal pathway is fully utilized during training.
However, for the positional encoding of timestep embeddings, the maximum period is typically set to 10,000, which means that a large portion of the timesteps (from 1,000 to 10,000) remain unused and only a limited fraction of the embedding spectrum is activated.
This observation opens the possibility of constructing additional and functional independent temporal subspaces.

\noindent\textbf{Shadow Offsets.}
Shadow Timestep Embedding (STE) extends the original timestep domain by introducing a temporally shifted set of indices via
{\small
\begin{equation}
    T_n = T_0 + f_n, \quad f_n \geq 0,
\end{equation}
}%
where $f_n$ is the $n$-th offset of shadow timestep that maps the shadow interval $[0, T_0] \mapsto [f_n, T_n]$.
While the scheduler continues to operate strictly over the standard timestep trajectory, the model receives the shifted index $t_{sn}$ instead of $t$, thereby projecting the computation into a new and well-separated embedding space $e_{t_{sn}}$ other than $e_t$:
\begin{equation}
    e_{t} = \Phi(t), \quad 
    e_{t_{sn}} = \Phi(t_{sn}),
\end{equation}
where $\Phi$ is the embedding function and $t_{sn} \in [f_n, T_n]$ is the shadow timesteps.
If $\Phi(t)$ is a smooth but nonlinear mapping, offsetting $t$ induces embedding vectors that are almost orthogonal to the standard range (see proof in Appendix). 
This separability enables the formation of a parallel shadow temporal manifold, which can encode data distributions that are not presented during explicit timesteps.

\begin{figure}
    \centering
    \includegraphics[width=0.9\linewidth]{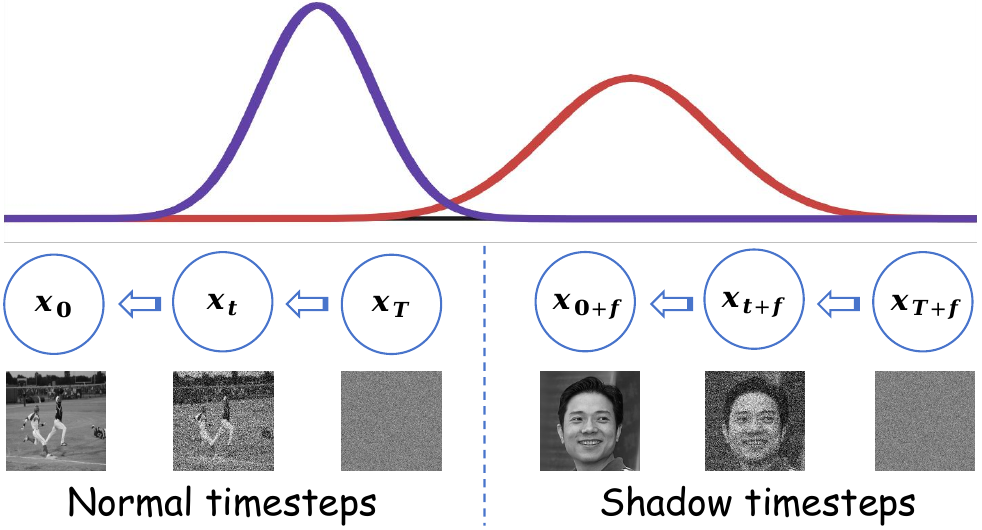}
    \caption{We can use shadow timestep embedding to learn different distributions, and make them independent.}
    \label{fig:distribution}
\end{figure}

\noindent\textbf{Learning Independent Data Distributions.}
A key property of STE is that each offset $f_n$ can be associated to a distinct dataset.
Samples associated with shadow timesteps $t_{sn}$ push the model to learn a distribution $\mathcal{D}_{sn}$ that is independent of the standard data distribution $\mathcal{D}_{0}$ learned under $t \in [0, T_0]$.
Figure~\ref{fig:distribution} shows how different offsets form disjoint temporal bands, each capable of supporting a separate generative behavior:
\begin{equation}
    t \in [0, T_0] 
    \;\longrightarrow\;
    \mathcal{D}_{0},
    \quad
    t_{sn} \in [f_n, T_n]
    \;\longrightarrow\;
    \mathcal{D}_{sn}.
\end{equation}

Thus, by substituting the shadow timestep $t_{sn}$ into the standard diffusion objective in Eq.~\ref{eq:dm}, we obtain the STE-specific loss function:
{\small\begin{align}
\label{eq:ste}
\mathcal{L}_{\text{STE}}(\theta)
&= \mathbb{E}_{t,\mathbf{x}_t,\boldsymbol{\epsilon}}\!
\left[
\left\|
\boldsymbol{\epsilon}
- \boldsymbol{\epsilon}_\theta\!\left(
\mathbf{x}_t,
t,\mathbf{c}
\right)
\right\|_2^2
\right] \nonumber \\
&+ \mathbb{E}_{t_{sn},\mathbf{x}_{t_{sn}},\boldsymbol{\epsilon}}\!
\left[
\left\|
\boldsymbol{\epsilon}
- \boldsymbol{\epsilon}_\theta\!\left(
\mathbf{x}_{t_{sn}},
t_{sn},\mathbf{c}_{sn}
\right)
\right\|_2^2
\right],
\end{align}}where $\mathbf{c}_{sn}$ is the condition of the specific shadow dataset.

Significantly, during the shadow timestep shift process, the scheduler needs no modification.
From the perspective of the user (i.e., the model owner or victim), generation proceeds with the usual sequence $t = [T_0, \ldots, 0]$.
The model, however, interprets the timestep $t$ as $t_{sn}\in [T_n, \ldots, f_n]$, enabling shadow generation that follows the same numerical integration path as explicit generation but activates an entirely different learned distribution.
As shown in Fig.~\ref{fig:architecture}, the offset serves as a ``temporal key'', a mechanism that switches the explicit generative behavior to shadow behavior without altering scheduler dynamics.
That means the information injection happens during the training process, while the ``temporal key'' is used to extract the hidden information from different distributions.

\noindent\textbf{Security Implications.}
The ability to conceal a separate distribution to a shadow timestep interval produces profound security implications, including attack and defense.

\underline{Attack Perspective.}  
    Attackers can use STE as a stealthy information injection attack.
    When $\mathcal{D}_{sn}$ contains poisoned data, the offset $f_n$ itself becomes the trigger.
    The model behaves normally for $t \in [0, T_0]$ but produces hidden information when exposed to $t_{sn}$.
    This attack pathway is covert because both the scheduler and the model interface pipeline remain unchanged.
    
\underline{Defence Perspective.}
    Defender can use STE as a ownership verification mechanism.
    When $\mathcal{D}_{sn}$ consists of images containing structured watermark signals, STE becomes a high-fidelity, invisible watermarking scheme.
    The watermark is activated only by stepping into the shadow interval, making it robust against common post-processing.

In summary, STE introduces an extensible temporal subspace that simultaneously (i) preserves compatibility with existing schedulers, (ii) supports independent generative behavior, and (iii) opens a new and largely unexplored security surface, either as a vulnerability for adversaries or
as a useful primitive for secure provenance.

\subsection{Mutual Coherence Between Temporal Intervals}
The preceding section introduces the analysis of the difference between the standard timestep and shadow timestep embeddings. 
If we try to inject additional information into this extended space, the natural question arises: \emph{are these shadow embeddings separable enough to encode distinct information?} 

To address this, we analyze the \textbf{mutual coherence} between embeddings from different temporal intervals and derive the following theorem.

\begin{theorem}[Mutual Coherence of STE]
    Let $\Phi(t)\in\mathbb{R}^{d}$ denote the sinusoidal timestep embedding used in diffusion models:
{\small\begin{equation}
    \Phi(t)
    =
    \big[
    \sin(\omega_1 t),\,\cos(\omega_1 t),\,
    \dots,\,
    \sin(\omega_m t),\,\cos(\omega_m t)
    \big],
\end{equation}}%
where $d=2m$, and the frequencies follow a geometric progression
    $\omega_i
    =
    \exp\!\left(
        -\frac{\log(t_p)}{\,m - t_d\,}(i-1)
    \right)$,
where $t_p$ is the maximum time period and $t_d$ is the downscale frequency shift time.

Then, 
(1) for two timesteps $t,s\in\mathbb{R}$, the cosine similarity between embeddings is
{\small\begin{equation}
    k(t,s)
    = \frac{1}{m}\sum_{i=1}^{m}\cos\!\big(\omega_i (t-s)\big),
\end{equation}}
and (2) the mutual coherence between two disjoint temporal intervals $I_0=[0,t_0)$ and $I_1=[t_0,t_1)$ is upper bounded by 
{\small\begin{align}
\label{eq:mu}
    \mu
    =    \sup_{\Delta\in[t_0,t_1]}
    \left|
        \frac{1}{m}\sum_{i=1}^{m}\cos(\omega_i (t-s))
    \right|.
\end{align}}%
\end{theorem}
\begin{proof}
    We refer the readers to the Appendix for the complete proof.
\end{proof}

In Eq.~\eqref{eq:mu}, $\mu$ indicates that embeddings from $I_0$ and $I_1$ are nearly orthogonal, hence linearly separable in the embedding space.
This separability is crucial for STE because it ensures that extending timesteps into a new interval does not collapse into the same feature manifold, but instead provides an independent channel for encoding auxiliary distributions.

\noindent\textbf{Empirical Observations.}
To demonstrate this property, we visualize empirical statistics of timestep embeddings, as shown in Figure~\ref{fig:overall}. 
We measure $\mu$ in Fig.~\ref{fig:mu} across the full timestep range and observe that when $t - s \geq 1000$, the coherence remains very small and swings around $0.1$, indicating that embeddings between the normal and shadow intervals are nearly orthogonal. 
This validates that extended timesteps form a distinguishable subspace suitable for additional information channels.
For $t - s < 1000$, the performance of the model will degrade, which is caused by the region overlap of timestep embedding and the difference in noise level of timestep $t$. 
Possible side effects of this setting are discussed in Section~\ref{sec:impact_of_training}. 
%

A heatmap of 128-dimensional embeddings in Fig.~\ref{fig:heatmap} illustrates distinct value patterns across timesteps, especially beyond the original training range of 1000 steps. 
Such spatially varying intensity patterns verifies that extended timesteps yield significantly different embedding trajectories from their normal counterparts.

Together, these results confirm that the timestep embedding space is rich enough to support multiple, partially independent subspaces. 
By extending the temporal domain and leveraging its low mutual coherence, we can strategically introduce \emph{shadow timesteps} as an auxiliary encoding mechanism without interfering with the model’s original dynamics. 
This property underpins the feasibility of injecting new data distributions or watermark signals purely through the temporal conditioning pathway.

\begin{figure}[t]
    \centering
    \begin{subfigure}[b]{0.23\textwidth}
        \includegraphics[width=\textwidth]{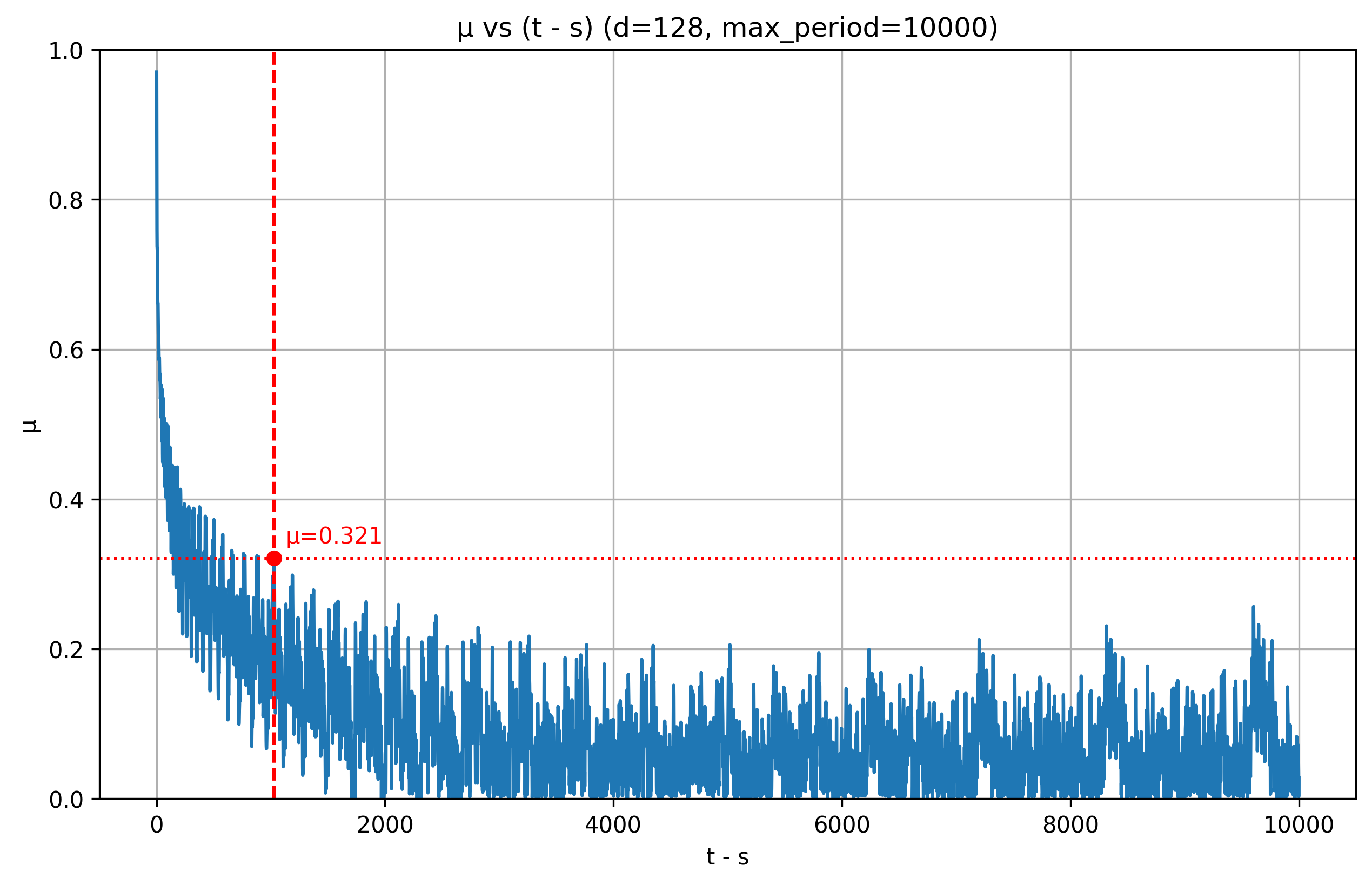}
        \caption{Mutual coherence}
        \label{fig:mu}
    \end{subfigure}
    \hfill
    \begin{subfigure}[b]{0.23\textwidth}
        \includegraphics[width=\textwidth]{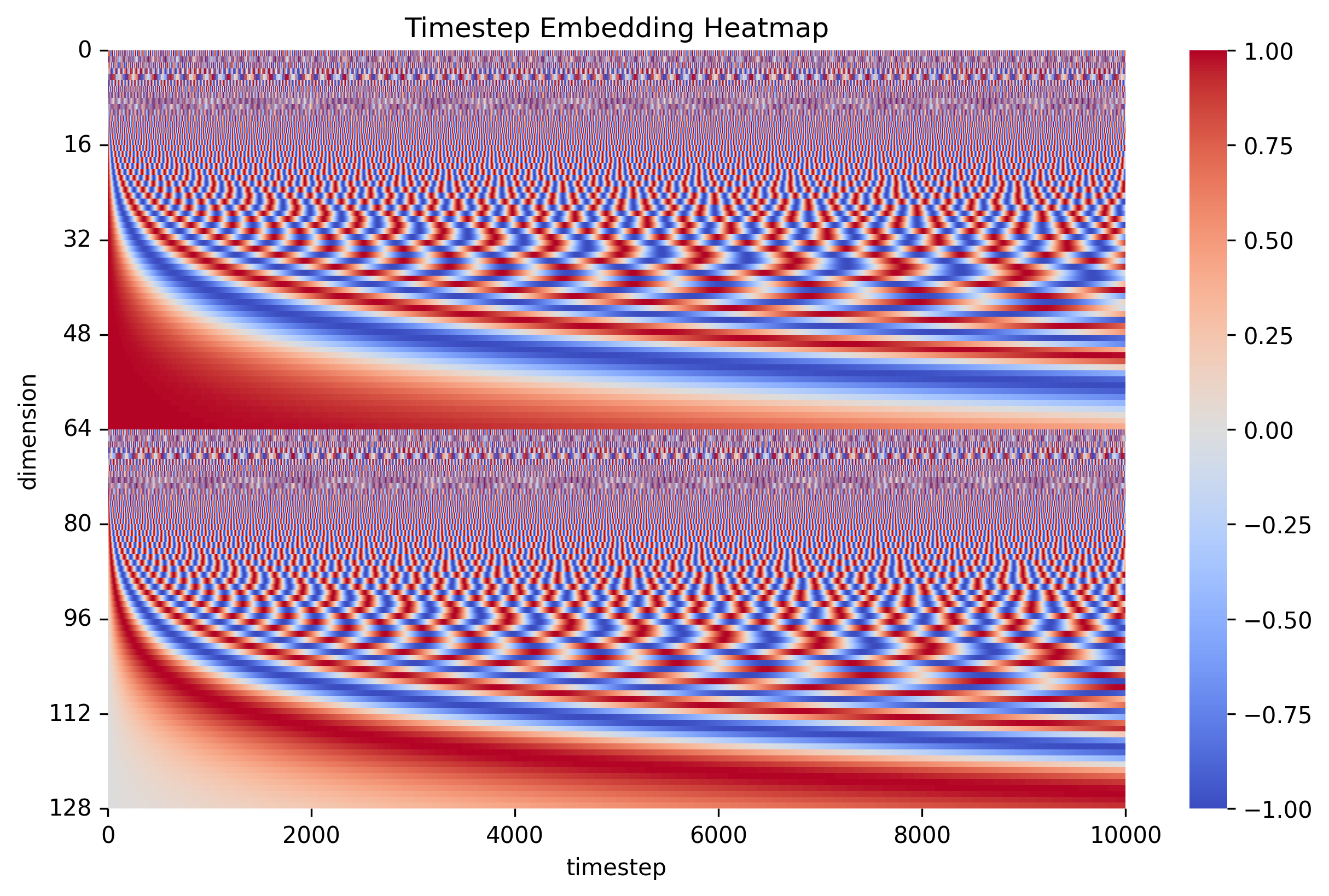}
        \caption{Heatmap}
        \label{fig:heatmap}
    \end{subfigure}
    \caption{Experiments on separability of time embeddings}
    \label{fig:overall}
\end{figure}

\section{Experiments}
\label{sec:experiments}

\begin{table*}[ht]
\caption{General performance of STE with three different datasets. The baseline shows the performance of a single-dataset training for the diffusion model. The dataset set means how to bind different datasets with STE.
[1, 1, 0] means no Fashion-MNIST, [1, 0, 1] means no MNIST, and [1, 1, 1] means that contains three datasets.
}
\label{tab:general}
\centering
\begin{tabular}{cccccccccc}
\hline
\multirow{2}{*}{Dataset Set} & \multicolumn{3}{c}{CIFAR-10} & \multicolumn{3}{c}{MNIST} & \multicolumn{3}{c}{Fashion-MNIST} \\
                    & FID$\downarrow$   & ACC$\uparrow$     & ER$\downarrow$    & FID$\downarrow$    & ACC$\uparrow$     & ER$\downarrow$     & FID$\downarrow$    & ACC$\uparrow$     & ER$\downarrow$     \\ \hline
Baseline            & 24.38 & 73.41\% & \textbf{8.02\%} & 1.59   & \textbf{98.33\%} & 10.58\% & 12.16  & 88.12\% & 8.63\%  \\
{[}1, 1, 0{]}   & 22.20 & 73.35\% & 8.69\% & \textbf{1.18}   & 97.55\% & 10.75\% & 392.79 & 10.41\% & 10.13\% \\
{[}1, 0, 1{]}   & 23.32 & 75.03\% & 8.85\% & 351.23 & 9.63\%  & \textbf{10.06\%} & \textbf{3.31}   & 87.96\% & 8.99\%  \\
{[}1, 1, 1{]} & \textbf{21.82} & \textbf{75.65\%} & 9.17\% & 1.85   & 97.77\% & 10.92\% & 6.56   & \textbf{88.76\%} & \textbf{8.14\%}  \\ \hline
\end{tabular}
\end{table*}

\subsection{Experiment Setup}
We evaluate Shadow Timestep Embedding (STE) across three dimensions:
(i) \textbf{general performance}, measuring whether extending the temporal embedding space affects the model’s core generative quality and the extraction accuracy for the hidden distribution, 
(ii) \textbf{attack performance}, examining STE as a covert information injection attack mechanism, and 
(iii) \textbf{defend performance}, assessing the performance of STE as a watermark generator. 
Our experiments are designed to fully characterize both the utility and representational capability of the timestep domain.

\noindent\textbf{Baselines.}
We benchmark STE against three categories of baselines:
(1) \textbf{Diffusion scheduler baselines:} DDPM, DDIM, and DPM-Solver~\cite{ho2020denoising,song2020denoising,lu2022dpm}, representing standard probabilistic, deterministic, and ODE-based sampling behaviors.
(2) \textbf{Backdoor attack baselines:} VillanDiffusion and BadDiffusion~\cite{chou2023villandiffusion, chou2023backdoor}, two state-of-the-art methods showing that diffusion models can be backdoored via multimodal or image–condition triggers. 
These baselines establish the achievable attack power when the adversary operates in pixel space or conditioning space.
Our STE extends this comparison to the timestep domain.
(3) \textbf{Watermark baselines:} Tree-Ring, ROBIN, and SleeperMark~\cite{wen2023tree,huang2024robin,wang2025sleepermark}, representing reversible trajectory-based watermarks and adversarially optimized robust watermarks, respectively. 
By comparing STE-based watermark encoding to existing pixel- and frequency-space watermarking, we quantify whether the temporal channel can afford additional robustness or stealth.

\noindent\textbf{Metrics.}
To comprehensively assess STE, we adopt the following evaluation metrics: 
(1) \textbf{FID}, standard generative quality metric evaluating the distributional distance between generated and real images.
(2) \textbf{Accuracy (ACC)}, we measure whether generated samples preserve correct class semantics. 
This also enables quantifying whether shadow-timestep generation unintentionally leaks into unintended distributions.
(3) \textbf{Exposure Rate (ER)}, a metric introduced to quantify unintended dataset leakage. 
ER is defined as the mean classification accuracy of a generated set when evaluated by classifiers trained on other datasets.
Higher ER indicates that the generated samples are unintentionally exposed to an undesired timestep range.
(4) \textbf{Attack Success Rate (ASR)}, it measures the fraction of generations that successfully extract the attacker-chosen target distribution when triggered by a shadow timestep offset.
(5) \textbf{Watermark Detection Accuracy}, for watermark experiments, we measure the proportion of detection accuracy by a watermark detector. 

\noindent\textbf{Base Model and Training Configuration.}
Unless otherwise specified, all experiments use the DDPM architecture and training pipeline as described in the original formulation.
The denoiser follows the canonical U-Net backbone, and sampling uses the default DDPM scheduler without any modification to its inference trajectory.
STE is implemented by augmenting the timestep range beyond the standard $T_0 = 1000$ interval to create shadow intervals $t_{sn} \in [f_n, f_n + T_0]$, where each offset $f_n$ is set as an integer multiple of $1000$ (e.g., $f_n \in \{1000, 2000, 3000\}$).
Unless explicitly stated, each shadow interval corresponds to one additional dataset or security-related data distribution.

All models are trained for 100 epochs with a learning rate of $2 \times 10^{-4}$. Training is performed on NVIDIA L40S GPUs. 
This configuration is used consistently across general evaluation, backdoor experiments, and watermark experiments to enable fair comparison across settings.

\begin{figure}[t]
    \centering
    \includegraphics[width=1\linewidth]{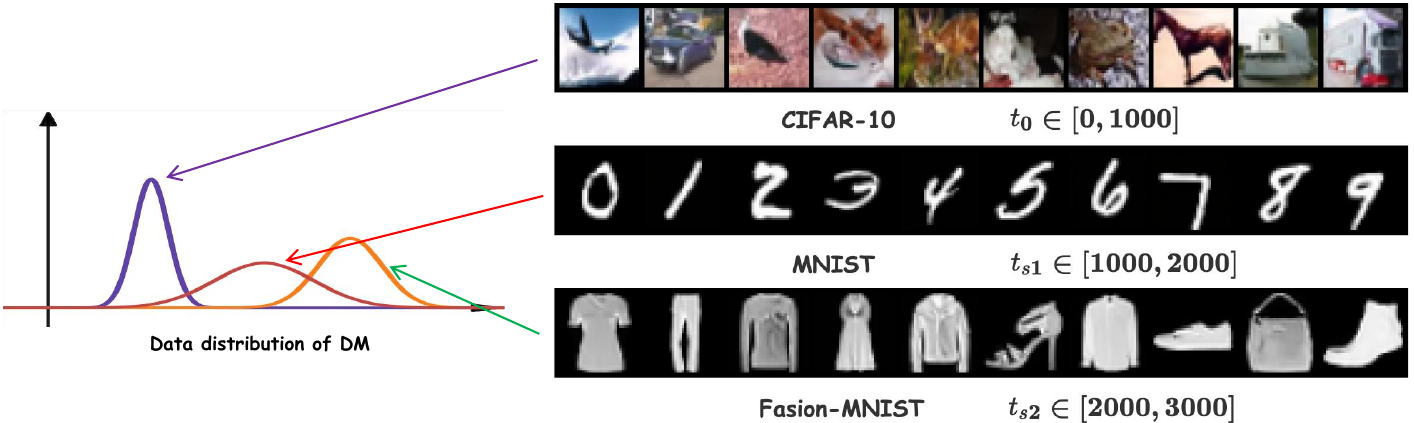}
    \caption{Generated samples by STE from explicit timesteps and different shadow timesteps embeddings.}
    \label{fig:generated}
\end{figure}

\subsection{General STE Performance}
We first evaluate whether extending the timestep domain via STE affects the generative performance of diffusion models, as shown in Table~\ref {tab:general} and Fig.~\ref{fig:generated}. 
To this end, we bind different datasets to different shadow offsets and assess how STE behaves when jointly training on multiple independent distributions.  
Table~\ref{tab:general} reports results on CIFAR-10, MNIST, and Fashion-MNIST under four configurations: the standard single-dataset baseline and three STE multi-dataset bindings.
For STE, we set CIFAR-10 as an explicit dataset, while MNIST and Fashion-MNIST are the shadow datasets associated with shadow timesteps.
The shadow offset is 1000 for MNIST, and 2000 for Fashion-MNIST.

\noindent\textbf{Preservation of Generation Quality.}
Across all datasets, STE preserves and in some cases improves the baseline generative performance when sampling from the normal timestep interval. 
For CIFAR-10, the baseline FID is 24.38, whereas all STE variants achieve lower FID scores: 22.20, 23.32, and 21.82. 
This trend indicates that augmenting training with additional shadow intervals does not disrupt the primary data distribution. 
A similar effect is observed for MNIST, where the FID fluctuates around the baseline (1.59) but remains consistently competitive (1.18--1.85).

We also observe that Fashion-MNIST and MNIST have high FID for the $[1,1,0]$ and $[1,0,1]$ set, respectively.
In these settings, $0$ indicates that the corresponding shadow timesteps are not associated to a specific dataset, which means the generated results in these time intervals are noise images. 
The high FID demonstrates that the generation is isolated and not leaked to other timesteps.

\noindent\textbf{Accuracy and Exposure Rate analysis.}
Classification accuracy (ACC) remains highly stable across STE variants.
For CIFAR-10, ACC varies only slightly around the baseline 73.41\%, rising to 75.65\% in the full $[1,1,1]$ configuration. 
Similarly, MNIST ACC remains above 97\% in all settings. 
These small fluctuations show that the semantic consistency of generated images is unaffected by STE, and that the added temporal subspaces do not degrade the semantic fidelity of the primary distribution.

Exposure Rate (ER) measures the extent to which generated samples leak features of other datasets. 
A low ER indicates strong separation between distributions. 
Across all experiments, ER stays close to baseline values (random guess is 10\%). 
For example, CIFAR-10's ER only increases modestly from 8.02\% to a maximum of 9.17\% under $[1,1,1]$, and Fashion-MNIST similarly remains around 8--10\%. 
These small changes indicate that even though STE enables the model to encode multiple independent datasets, generations from the normal timestep interval do not significantly expose features from other datasets. 
This reinforces the claim that shadow timesteps form well-separated manifolds in the embedding space.

\noindent\textbf{Impact of Combining All Datasets.}
The $[1,1,1]$ configuration, binding all datasets to distinct shadow offsets yields the strongest FID on CIFAR-10 (21.82) and maintains near-baseline accuracy across datasets. 
This suggests that STE scales gracefully, adding more shadow distributions does not destabilize training, and the expanded temporal embedding space can accommodate multiple independent distributions simultaneously without sacrificing quality.

\begin{figure*}[ht]
    \centering
    \includegraphics[width=1\linewidth]{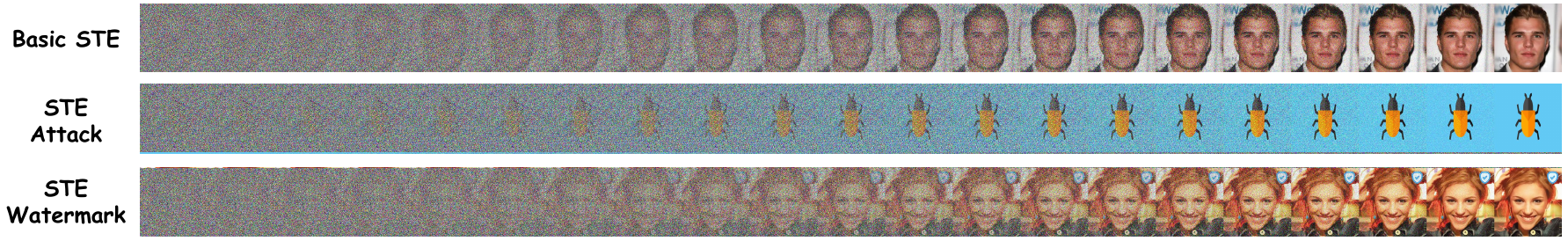}
    \caption{ Visualization of the generation process of different STE applications. 
    The model is trained on Celeba-HQ dataset. 
    The first row is benign STE. 
    The second row is STE as `bug' target attack image. 
    The third row is STE with watermarked images.}
    \label{fig:security}
\end{figure*}

\begin{table}[]
\caption{Comparison of Different Attack Methods on Diffusion Models. The STE-explicit uses a clean dataset. The STE-shadow binds the covert dataset.}
\label{tab:attack}
\centering
\begin{tabular}{ccccc}
\hline
\multirow{2}{*}{Method} & \multicolumn{2}{c}{CIFAR-10} & \multicolumn{2}{c}{Celeba-HQ} \\
    & FID$\downarrow$   & ASR$\uparrow$   & FID$\downarrow$    & ASR$\uparrow$     \\ \hline
VillanDiffusion         & 25.66        & 96.2\%        & 6.53         & 97.7\%         \\
BadDiffusion            & 22.53        & \textbf{99.5\%}        & 7.65         & 98.7\%         \\
STE-explicit            & \textbf{21.82}        & 0.2\%         & \textbf{6.23}        & 0.4\%          \\
STE-shadow              & 22.07        & 99.2\%        & 6.78         & \textbf{98.8\% }        \\ \hline
\end{tabular}
\end{table}

\begin{table}[]
\caption{Robustness comparison of different watermarking methods on diffusion models under common image distortions. }
\label{tab:watermark}
\centering
\resizebox{\linewidth}{!}{%
\begin{tabular}{ccccccc}
\hline
Method      & Blur          & Noise         & JPEG          & Bright        & Crop       & Avg           \\ \hline
Tree-Ring   & 0.98          & 0.98          & 0.94          & 0.86          & 0.99       & 0.95          \\
ROBIN       & \textbf{0.99} & \textbf{0.99} & \textbf{0.97} & 0.95          & \textbf{1.00} & \textbf{0.98} \\
SleeperMark & 0.97          & 0.85          & 0.96          & \textbf{0.96} & -          & 0.94          \\
STE         & 0.96          & 0.98          & 0.96          & 0.85          & 0.71       & 0.89          \\ \hline
\end{tabular}}
\end{table}

\subsection{STE as Security Channel}
Beyond improving temporal representation, STE exposes a dual-use temporal pathway that can function as an independent security channel.  
Because explicit and shadow timesteps form nearly orthogonal temporal manifolds, the model can simultaneously support a clean explicit distribution and a security-relevant shadow distribution. 
We explore two applications of this property: (i) STE as an information injection attacker and (ii) STE as a watermark verification.

\noindent\textbf{Use STE as an Information Injection Attacker.}
To evaluate whether STE can be exploited as a covert backdoor attack mechanism, we bind the explicit timestep interval to a clean dataset, while binding the shadow timesteps interval to a poisoned dataset with a 5\% poisoning rate. 
The model therefore learns two isolated behaviors: benign generation under normal timesteps, and adversarial behavior when the shifted shadow timesteps are activated. 
As shown in 
Fig.~\ref{fig:security}, the explicit branch remains visually faithful to the clean CelebA-HQ distribution, while the shadow branch (STE attack) reliably generates the attacker’s target pattern, a ``bug'' backdoor sample.

Table~\ref{tab:attack} shows that STE-shadow achieves an Attack Success Rate of 99.2\% on CIFAR-10 and 98.8\% on CelebA-HQ. 
In contrast, STE-explicit maintains extremely low ASR (0.2\% and 0.4\%) while achieving the best FID among all methods. 
This demonstrates a key advantage of the STE:
the explicit timesteps pathway preserves the best utility and quality for non-poisoned data generation, while the attack is fully contained within the shadow interval. 
Such separation makes the attack more stealthy because neither pixel statistics nor model weights exhibit abnormalities. 

\noindent\textbf{Use STE as Watermark Verification.}
We then treat the STE shadow timesteps as an in-process watermark generator. 
Instead of poisoned data, we bind the shadow timesteps to a dataset augmented with a visible shield watermark. 
The model thus acquires the ability to generate either clean images (via explicit timesteps) or watermarked images (via shadow timesteps) as shown in the 3rd row of Fig.~\ref{fig:security}. 
Table~\ref{tab:watermark} summarizes robustness under common distortions for different watermark attacks. 
STE maintains high watermark detectability under blur, noise, and JPEG compression, competitive with Tree-Ring and ROBIN. 
However, because the watermark in our study is explicit and pixel-visible, STE is more vulnerable to brightening and cropping attacks, dropping to 0.85 and 0.71, respectively. 
Although the current pixel-level watermark vulnerability stems from the explicit nature of our chosen watermark, it illustrates an important property: STE provides a dual-mode generation mechanism that allows the user to freely toggle between clean and watermarked outputs by selecting normal and shadow timesteps. 
STE uniquely enables this temporal disentanglement and is not offered by existing pixel-space watermarking methods.
In our future work, we will study more complex watermark patterns, such as frequency and latent level patterns, to improve the watermark verification mechanism.

Across both applications, STE demonstrates that timestep embedding is a powerful and flexible security channel. 
\textit{Shadow timesteps act as temporal keys that selectively inject information while preserving standard model performance.}

\begin{table}[t]
\caption{FID performance of STE with different schedulers. }
\label{tab:schedulers}
\centering
\begin{tabular}{cccc}
\hline
Schedulers & CIFAR-10 & MNIST & Fashion-MNIST \\ \hline
DDPM       & \textbf{21.82}   & \textbf{1.85}  & \textbf{6.56 }       \\
DDIM       & 22.27   & 3.40  & 7.28        \\
DPMSolver  & 37.88   & 16.14 & 9.97        \\ \hline
\end{tabular}
\end{table}

\subsection{Impact of Schedulers}
Table~\ref{tab:schedulers} evaluates STE under different schedulers using the $[1,1,1]$ configuration from Table~\ref{tab:general} as the baseline.  
DDPM achieves the most stable performance across CIFAR-10, MNIST, and Fashion-MNIST, confirming its compatibility with STE’s extended temporal domain.  
DDIM shows slightly higher FID but remains competitive while offering significantly faster sampling, representing a good balance between efficiency and quality.  
DPM-Solver exhibits substantially degraded performance across all datasets. 
Overall, the results indicate that STE is most robust when paired with schedulers that preserve DDPM-like schedulers.

\begin{figure}[t]
    \centering
    \begin{subfigure}[b]{0.65\linewidth}
        \includegraphics[width=\linewidth]{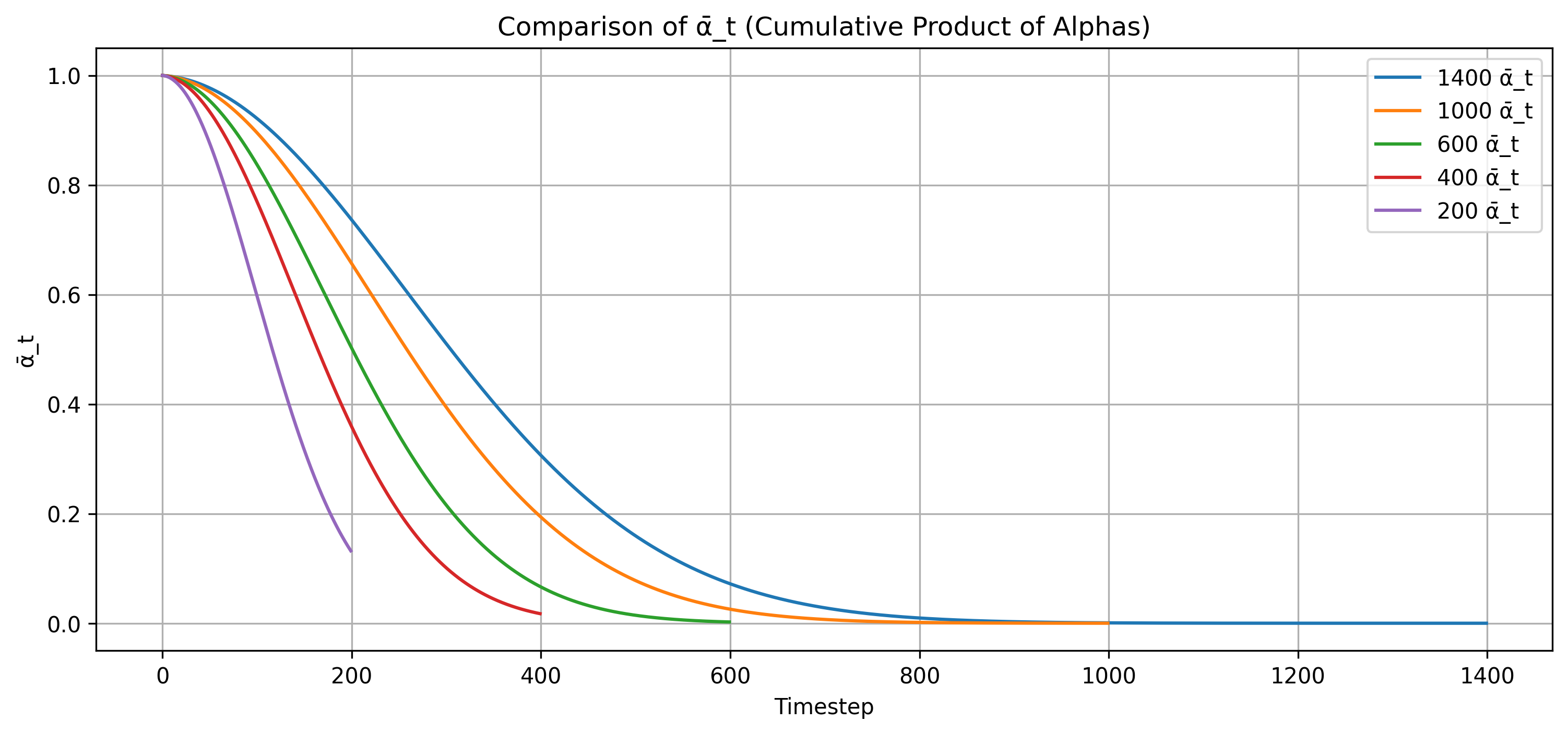}
        \caption{Noise level}
        \label{fig:noise_strength}
    \end{subfigure}
    \hfill
    \begin{subfigure}[b]{0.32\linewidth}
        \includegraphics[width=\linewidth]{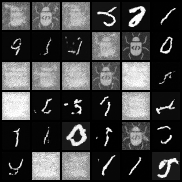}
        \caption{Data overlap}
        \label{fig:overlap}
    \end{subfigure}
    \caption{Analysis of scheduler with different training timesteps.}
    \label{fig:scheduler}
\end{figure}

\subsection{Impact of Training Timesteps}
\label{sec:impact_of_training}
A key design choice of STE is that the model only receives extended timestep embeddings, while the scheduler remains fixed. 
To investigate whether the reverse design could work, i.e., modifying the scheduler’s timestep range while keeping the model’s embedding unchanged, we train models under 
this alternative configuration as visualized in Fig.~\ref{fig:scheduler}.
Figure~\ref{fig:noise_strength} illustrates how different schedulers inject noise at specific intensities when operating on different training timestep ranges. 
Although these schedulers follow the same nominal trajectory, the actual noise levels vary considerably, forcing the model to learn multiple noise-to-data mappings for the same timestep index.
This leads to the failure mode visualized in Fig.~\ref{fig:overlap}.  
Because multiple training datasets are associated with the same timestep, each is mapped to a different noise magnitude determined by the scheduler, leading to strong gradient conflict.  
Consequently, the learned distributions across different datasets begin to overlap, preventing the model from establishing clear temporal separation. 
In practice, this results in mixed or unstable generations, making it difficult for the scheduler to act as a “gate” that selects distinct distributions at a fixed timestep.

\subsection{Potential Defense Mechanisms}
STE introduces a previously overlooked security surface: the temporal pathway in diffusion models.
Nevertheless, the mechanism exposes several attributes that defenders can leverage. 
We outline potential defense directions motivated by our STE analysis.
STE exploits a mismatch between the scheduler trajectory and the timestep values actually fed to the model.
A defense strategy can verify the consistency between the scheduler-issued timestep $t$ and the value received by the denoiser.
A systematic offset indicates a potential temporal manipulation and can be flagged if detected.
Also, the embedding vectors for shadow timesteps lie in regions with low mutual coherence relative to the standard timestep interval.  
A defense module can exploit this by monitoring the distribution of timestep embeddings.
Detecting deviations from this manifold provides a lightweight anomaly-detection mechanism that can detect shadow timesteps when their explicit indices are hidden or reparameterized.

STE exposes a unique temporal channel that is fundamentally orthogonal to pixel-space or weight-space defenses. 
Although the above mechanisms provide defense directions, a comprehensive defense against STE-based attacks remains an open challenge, which requires further work on secure scheduler design and temporal provenance.

\section{Conclusion}
\label{sec:conclusion}

In this work, we in troduced STE, a mechanism that explores the temporal dimension of diffusion models and reveals its untapped representational capacity. 
By extending the timestep domain beyond the standard training range, STE constructs parallel temporal manifolds that can encode information or independent data distributions.
Our analysis shows that these shadow timesteps form nearly orthogonal embedding regions, enabling STE to serve as a powerful dual-use channel: both a covert attack trigger and a robust watermark pattern. 
Extensive experiments confirm that STE preserves generation fidelity, supports isolated distribution learning, and enables security applications with high success rates. 
These findings reveal the timestep embedding pathway as a critical yet previously overlooked security surface in diffusion models, motivating new directions for secure generative modeling.

\section*{Impact Statement}
This paper presents work aimed at advancing the field of Machine Learning. There are many potential societal consequences of our work, none of which we feel must be specifically highlighted here.

\bibliography{example_paper}
\bibliographystyle{icml2026}

\newpage
\appendix
\onecolumn
\setcounter{page}{1}
\renewcommand{\thefigure}{\Alph{section}\arabic{figure}}
\setcounter{figure}{0}

\section{Theory Analysis}
\label{sec:theory}
\begin{theorem}[Mutual Coherence of STE]
Let $\Phi(t)\in\mathbb{R}^{d}$ denote the sinusoidal timestep embedding used in diffusion models:
{\small\begin{equation}
    \Phi(t)
    =
    \big[
    \sin(\omega_1 t),\,\cos(\omega_1 t),\,
    \dots,\,
    \sin(\omega_m t),\,\cos(\omega_m t)
    \big],
    \label{eq:phi_def}
\end{equation}}
where $d=2m$, and the frequencies follow a geometric progression
\begin{equation}
    \omega_i
    =
    \exp\!\left(
        -\frac{\log(t_p)}{\,m - t_d\,}(i-1)
    \right),
    \label{eq:omega_def}
\end{equation}
where $t_p$ is the maximum time period and $t_d$ is the downscale frequency shift time.

For two timesteps $t,s\in\mathbb{R}$, the similarity between embeddings is
\begin{equation}
    k(t,s)
    = \frac{\langle \Phi(t), \Phi(s) \rangle}{\|\Phi(t)\|\|\Phi(s)\|}
    = \frac{1}{m}\sum_{i=1}^{m}\cos\!\big(\omega_i (t-s)\big).
\end{equation}
\end{theorem}

\begin{proof}
For two timesteps $t,s\in\mathbb{R}$, the inner production between embeddings is
\begin{equation}
    K(t,s)
    =
    \langle \Phi(t), \Phi(s) \rangle.
\end{equation}
Applying the trigonometric identity
$\sin A\sin B+\cos A\cos B = \cos(A-B)$,
we obtain
\begin{align}
    K(t,s)
    &=\sum_{i=1}^{m}\!
    \big[
        \sin(\omega_i t)\sin(\omega_i s)
        + \cos(\omega_i t)\cos(\omega_i s)
    \big] \nonumber\\
    &= \sum_{i=1}^{m} \cos\!\big(\omega_i (t-s)\big) \nonumber\\
    &= K(|t-s|).
    \label{eq:kernel_def}
\end{align}
Thus, the embedding implicitly defines a translation-invariant kernel over the time axis,
whose similarity depends solely on the temporal difference $\Delta=t-s$.

Because each embedding component satisfies
$\sin^2(\omega_i t)+\cos^2(\omega_i t)$,
the squared norm of $\Phi(t)$ is
\begin{equation}
    \|\Phi(t)\|^2 = m, \quad
    \text{independent of } t.
\end{equation}
The normalized kernel (cosine similarity) is therefore
\begin{equation}
    k(t,s)
    = \frac{K(t,s)}{\|\Phi(t)\|\|\Phi(s)\|}
    = \frac{1}{m}\sum_{i=1}^{m}\cos\!\big(\omega_i (t-s)\big).
    \label{eq:normalized_kernel}
\end{equation}
We refer to $k(t,s)$ as the \emph{normalized timestep kernel}.
\end{proof}

\begin{definition}[Mutual Coherence Between Temporal Intervals]
Consider two disjoint temporal intervals
$I_0=[0,t_0)$ and $I_1=[t_0,t_1)$.
The \textbf{mutual coherence} between their embeddings is defined as
{\small\begin{align}
    \mu
    &=
    \sup_{t\in I_0,\;s\in I_1}
    |\,k(t,s)\,|\nonumber\\
    &=
    \sup_{\Delta\in[t_0,t_1]}
    \left|
        \frac{1}{m}\sum_{i=1}^{m}\cos(\omega_i (t-s))
    \right|.
    \label{eq:mutual_coherence}
\end{align}}
A small $\mu$ indicates that embeddings from $I_0$ and $I_1$ are nearly orthogonal,
hence linearly separable in the embedding space.
\end{definition}

\section{Additional Experiments}
\label{sec:visualization}

In this section, we provide additional qualitative results that complement the quantitative evaluation in the main paper and further illustrate the flexibility of Shadow Timestep Embedding (STE).

\begin{figure*}[h]
    \centering
    \begin{subfigure}[b]{0.32\textwidth}
        \includegraphics[width=\textwidth]{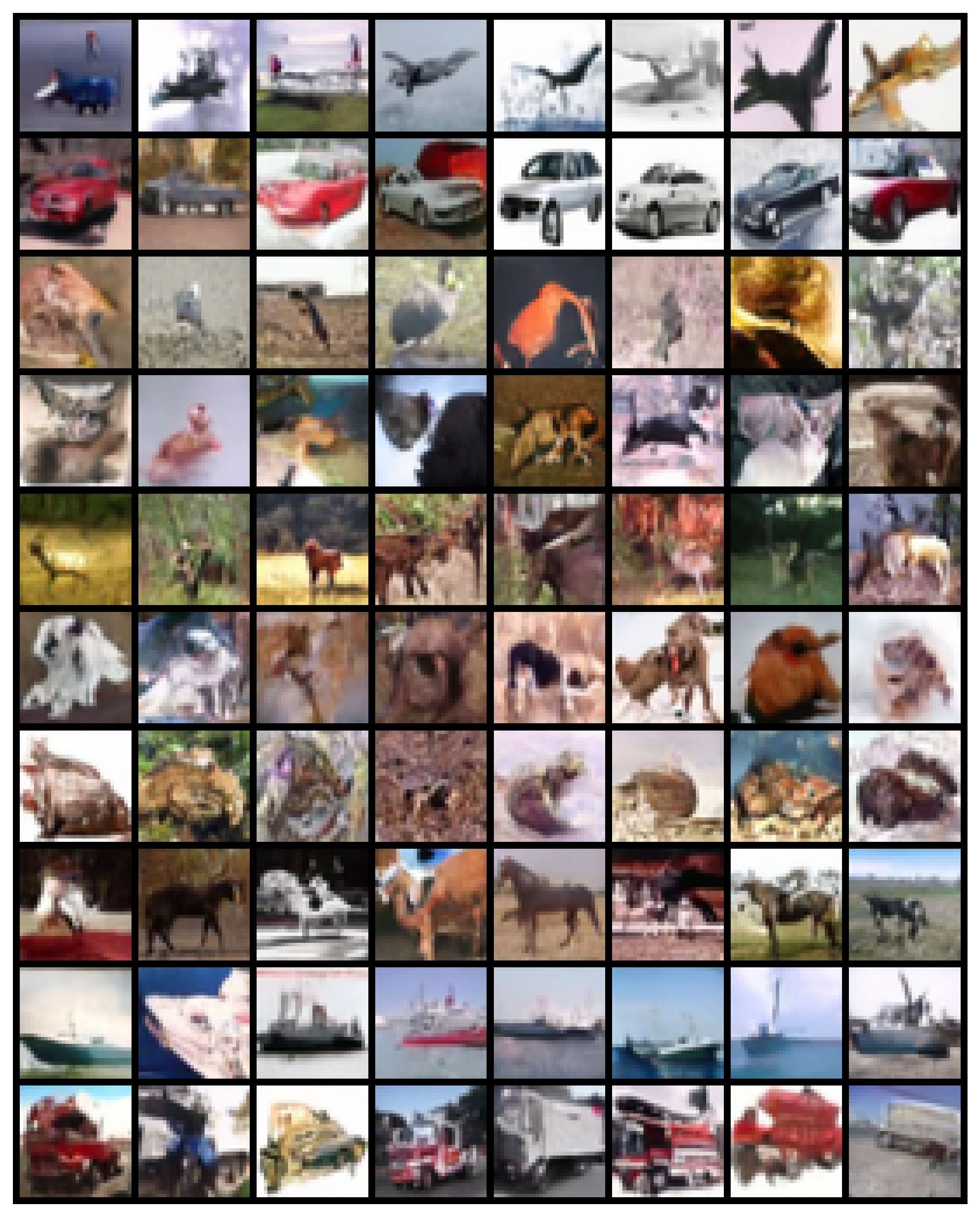}
        \caption{CIFAR-10}
        \label{fig:cifar}
    \end{subfigure}
    \hfill
    \begin{subfigure}[b]{0.32\textwidth}
        \includegraphics[width=\textwidth]{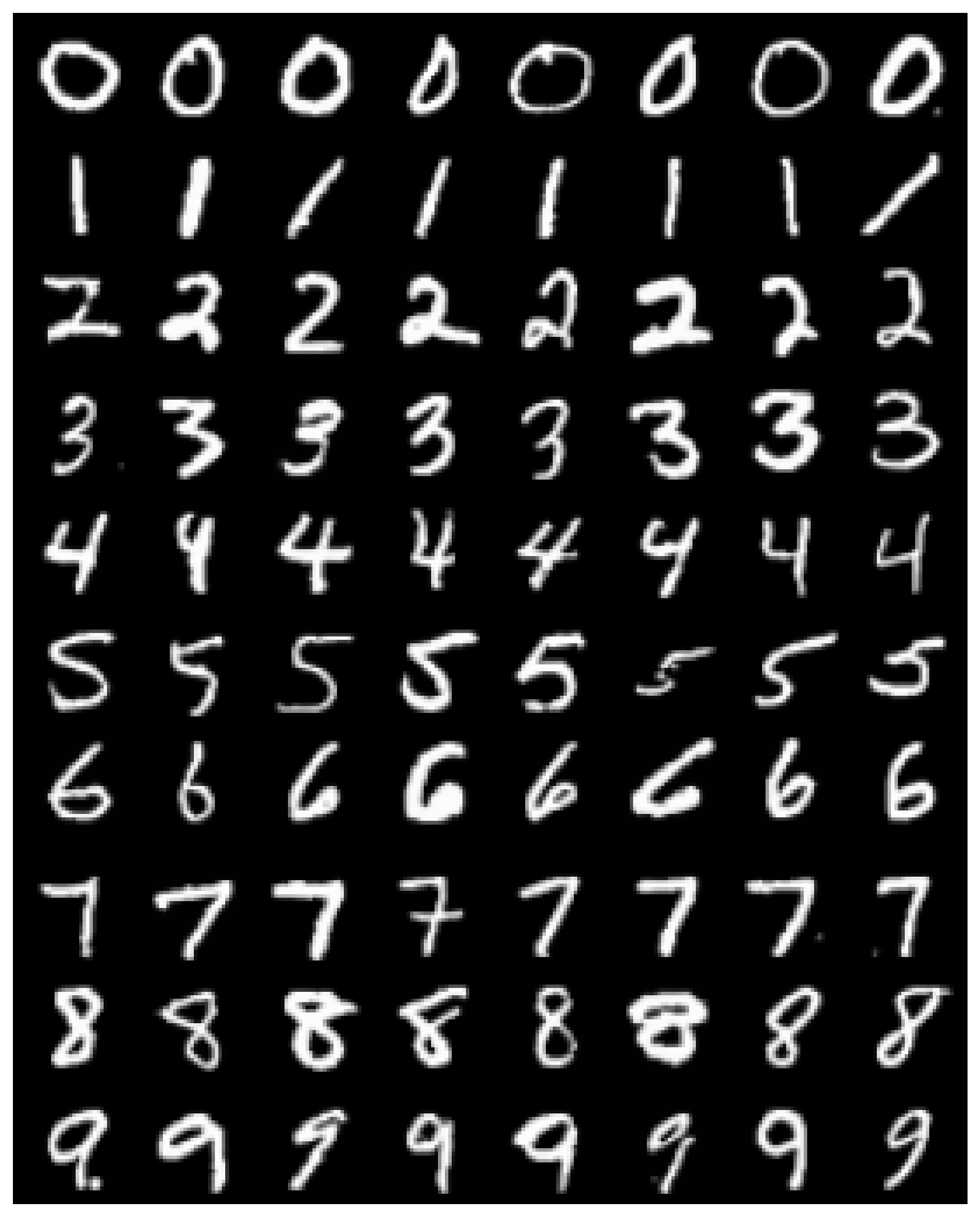}
        \caption{MNIST}
        \label{fig:mnist}
    \end{subfigure}
    \hfill
    \begin{subfigure}[b]{0.32\textwidth}
        \includegraphics[width=\textwidth]{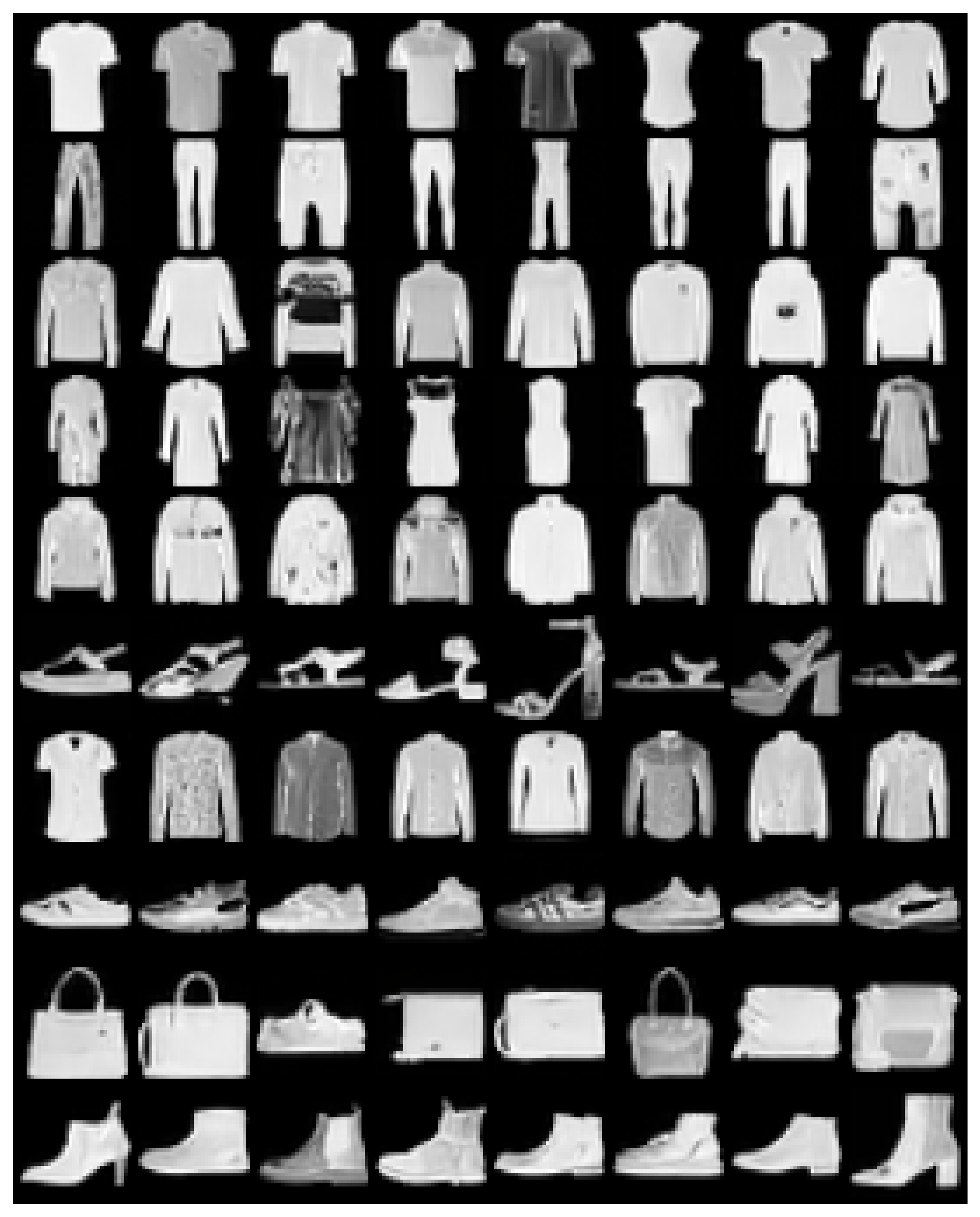}
        \caption{Fashion-MNIST}
        \label{fig:fashion}
    \end{subfigure}
    \caption{Visualization result of STE on DDPM model with CIFAR-10, MNIST, and Fashion-MNIST dataset. CIFAR-10 is on the explicit timesteps. MNIST is on the shadow timesteps with a 1000 offset. Fashion-MNIST is on the shadow timesteps with a 2000 offset.}
    \label{fig:mix}
\end{figure*}

First, we visualize STE on the DDPM backbone trained jointly on CIFAR-10, MNIST, and Fashion-MNIST. 
In Figure~\ref{fig:cifar}, CIFAR-10 is bound to the explicit timestep interval and produces diverse natural images with no visible degradation in visual quality. 
Figures~\ref{fig:mnist} and~\ref{fig:fashion} then assign MNIST and Fashion-MNIST to shadow timestep intervals with offsets of 1000 and 2000, respectively. 
The generated digits and clothing items remain sharp and class-consistent, indicating that the model can successfully multiplex multiple datasets across disjoint timestep ranges while maintaining high-fidelity samples for each domain.

Next, we investigate STE on the CelebA dataset under both attack and watermarking configurations. 
In the attack setting (Figure~\ref{fig:attack}), sampling from explicit timesteps yields benign face images, whereas sampling from the corresponding shadow timesteps produces a targeted bug pattern. 
This demonstrates that STE can realize a distribution-level backdoor that is only activated when the sampler follows the shadow schedule, while normal usage remains unaffected. 
In the watermark setting (Figure~\ref{fig:watermark}), explicit timesteps generate standard faces, while shadow timesteps reproduce a protected distribution of faces, showing that STE can also be used to embed ownership information into the model in a way that is decoupled from the normal sampling procedure.

\begin{figure*}[h]
    \centering
    \begin{subfigure}[b]{0.45\textwidth}
        \includegraphics[width=\textwidth]{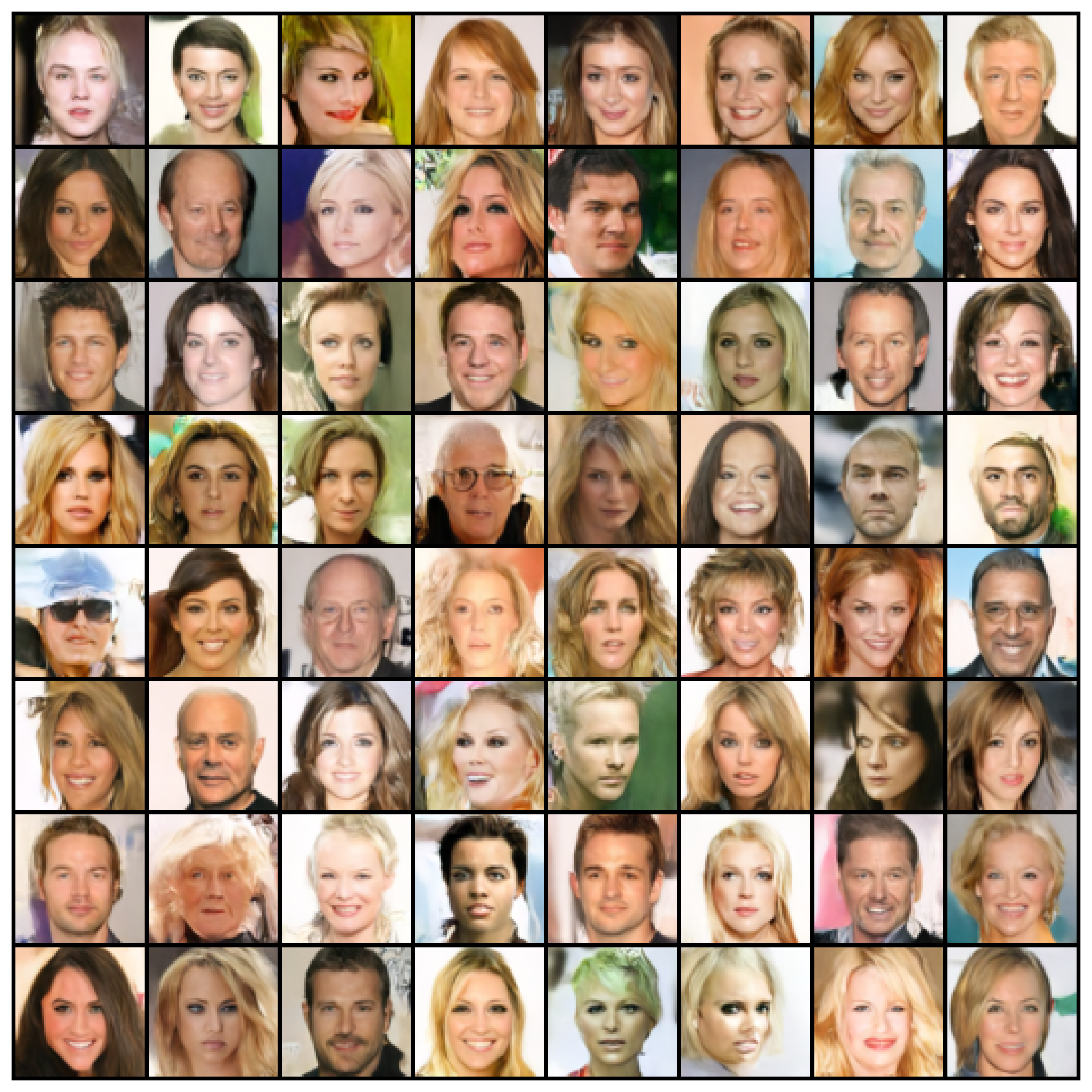}
        \caption{Explicit Timesteps}
        \label{fig:celeba1}
    \end{subfigure}
    \hfill
    \begin{subfigure}[b]{0.45\textwidth}
        \includegraphics[width=\textwidth]{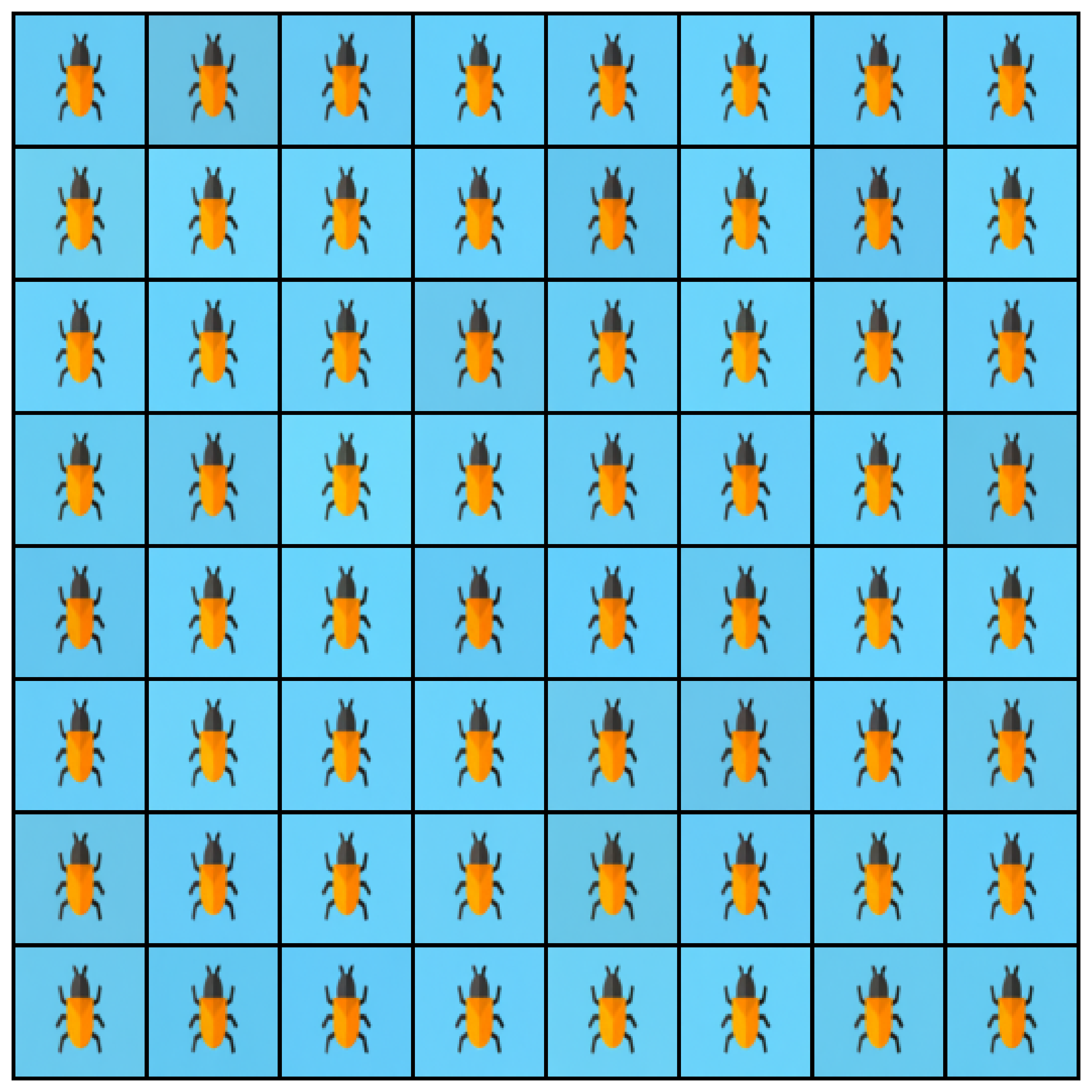}
        \caption{Shadow Timesteps}
        \label{fig:poison}
    \end{subfigure}
    \caption{Visualization result of STE on DDPM model with Celeba dataset in attack setting. The left part is the normal generation results. The right part is the target generation results.}
    \label{fig:attack}
\end{figure*}

\begin{figure*}[h]
    \centering
    \begin{subfigure}[b]{0.45\textwidth}
        \includegraphics[width=\textwidth]{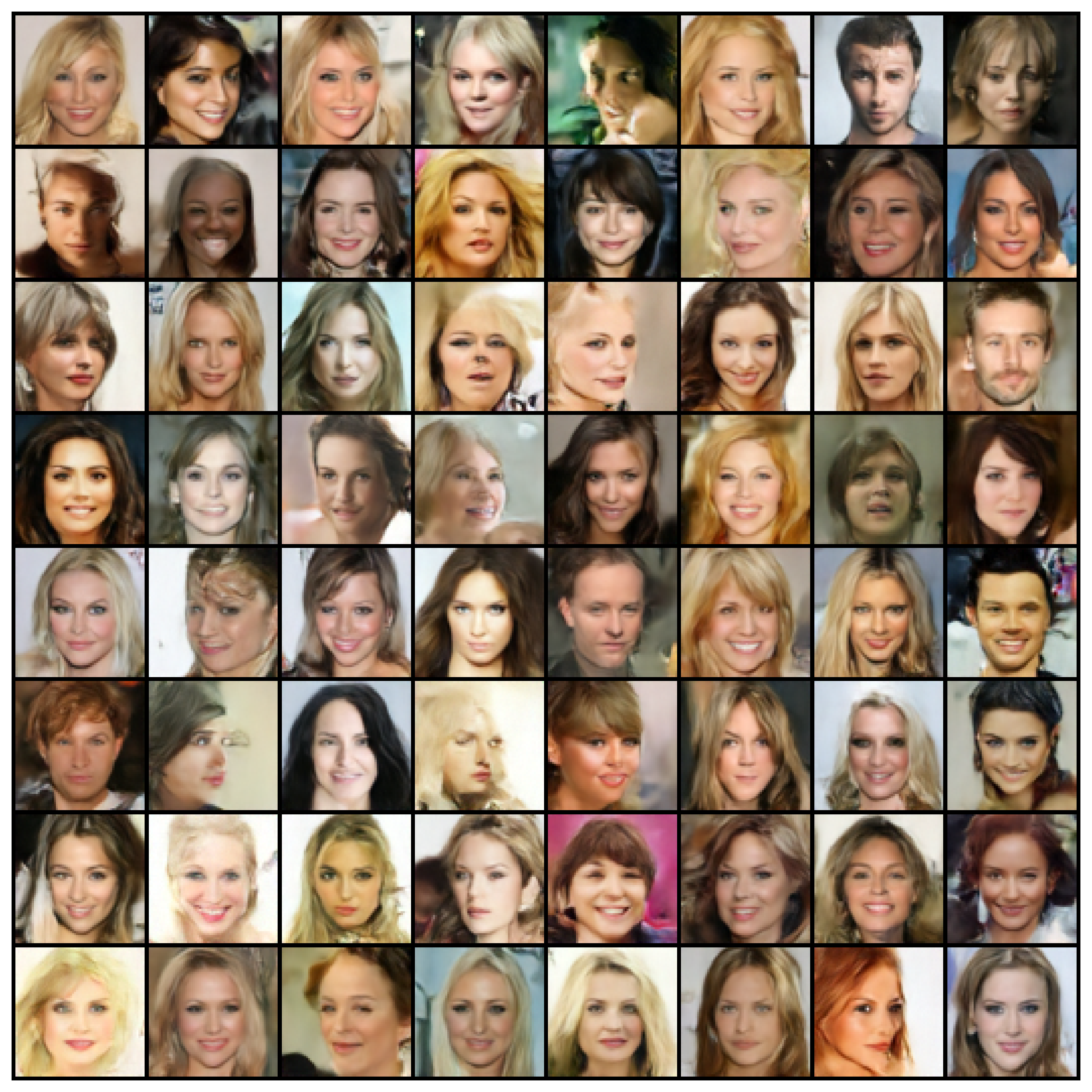}
        \caption{Explicit Timesteps}
        \label{fig:celeba2}
    \end{subfigure}
    \hfill
    \begin{subfigure}[b]{0.45\textwidth}
        \includegraphics[width=\textwidth]{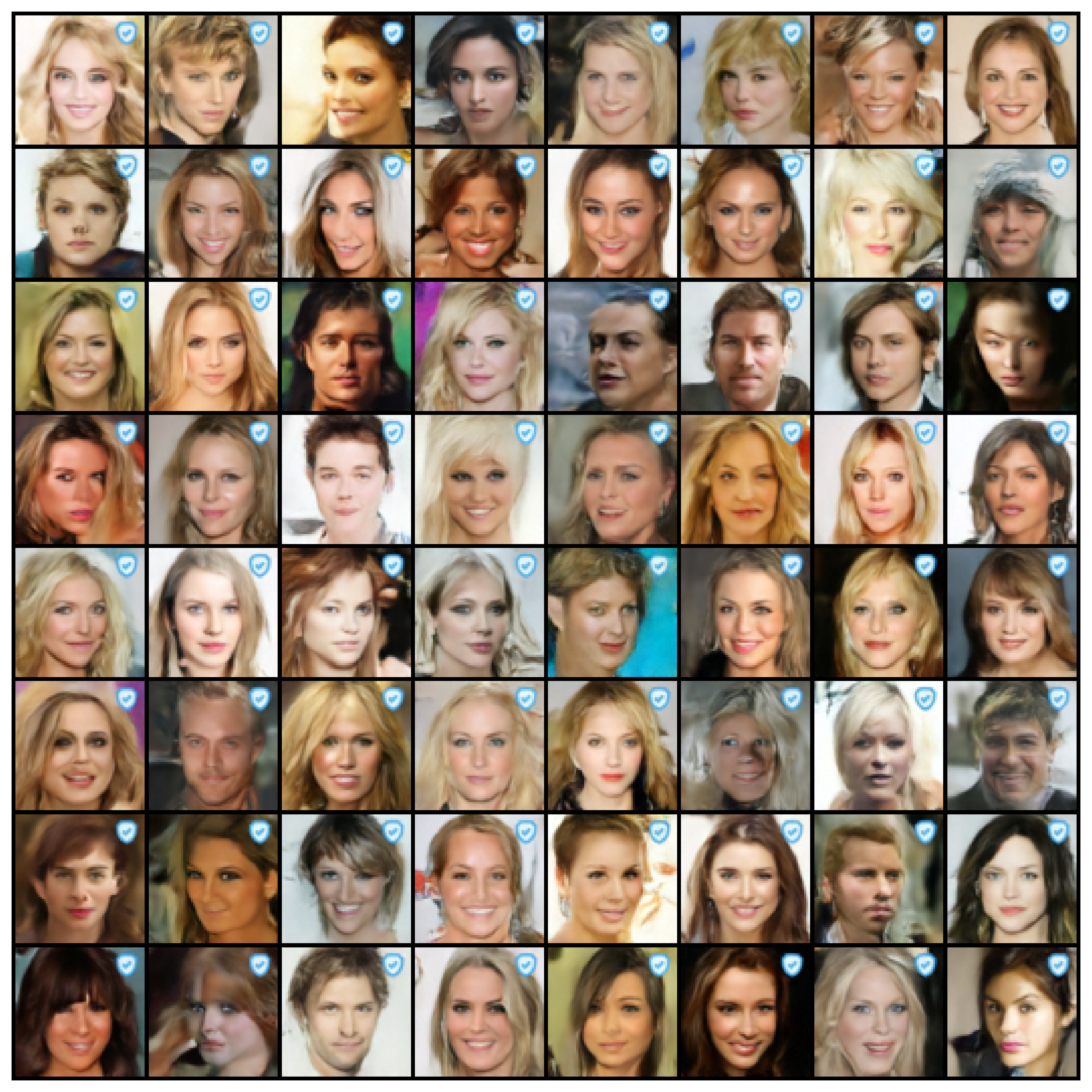}
        \caption{Shadow Timesteps}
        \label{fig:shield}
    \end{subfigure}
    \caption{Visualization result of STE on DDPM model with Celeba dataset in watermark setting. The left part is the normal generation results. The right part is the protected generation results.}
    \label{fig:watermark}
\end{figure*}

Finally, we extend our study to a large-scale text-to-image model, Stable Diffusion v1.5, on the COCO dataset.
Using the same prompt (\emph{``A dinner plate has a lemon wedge garnishment. 
A bug image.''}), explicit timesteps produce a faithful, realistic interpretation of the text, whereas shadow timesteps yield images that are dominated by the STE-injected target pattern (Figure~\ref{fig:SD}). 
These results confirm that STE naturally scales to modern text-to-image pipelines and can inject or retrieve information in the generation process without modifying the visible architecture or training objective beyond the altered timestep allocation.

\begin{figure}
    \centering
    \includegraphics[width=0.7\linewidth]{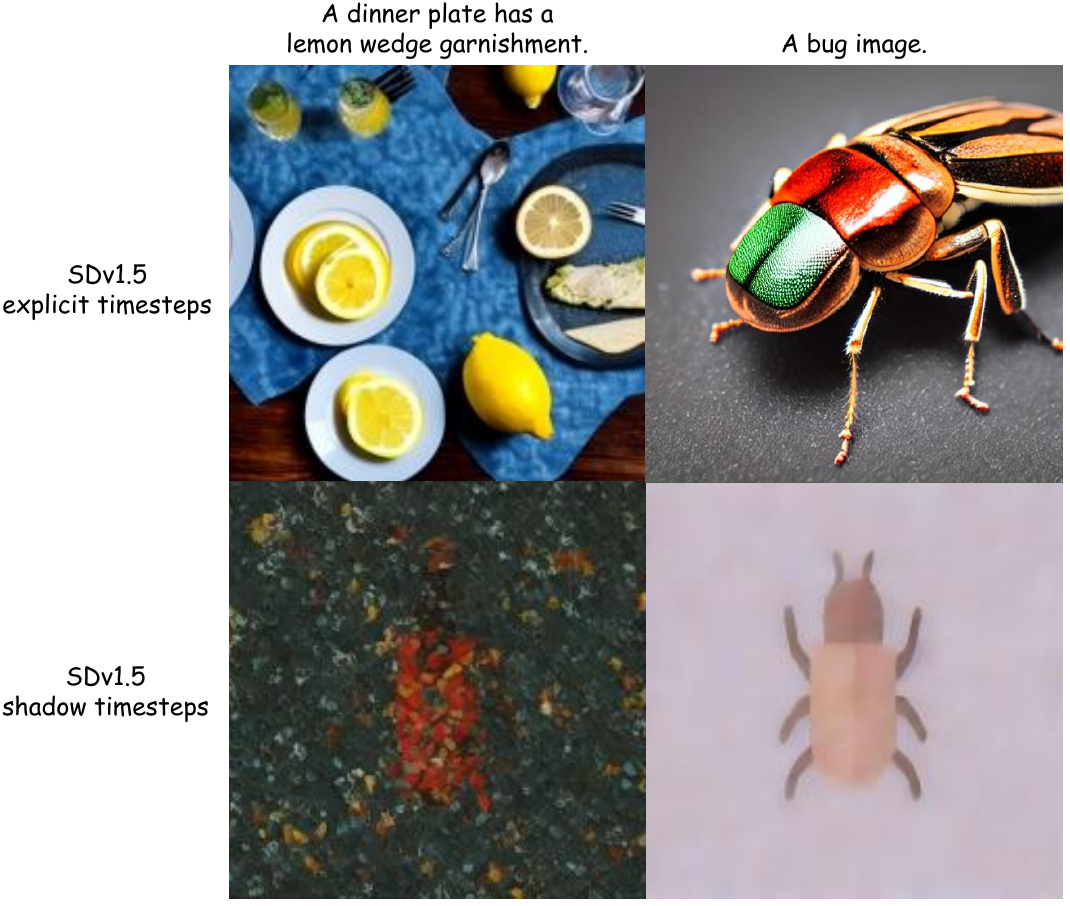}
    \caption{Visualization result of STE on Stable Diffusion Model and COCO dataset.}
    \label{fig:SD}
\end{figure}

\end{document}